%% file: main.tex
\documentclass{article}

\usepackage{graphicx}

    \usepackage[preprint, nonatbib]{neurips_2019}

\usepackage[utf8]{inputenc} 
\usepackage[T1]{fontenc}    
\usepackage{hyperref}       
\usepackage{url}            
\usepackage{booktabs}       
\usepackage{amsfonts}       
\usepackage{nicefrac}       
\usepackage{microtype}      
\usepackage{amsmath}
\usepackage{amsthm}
\usepackage{todonotes}
\usepackage{algorithm2e}

\usepackage{subcaption}
\usepackage{multirow}

\newtheorem{theorem}{Theorem}

\title{Sub-Goal Trees -- a Framework for Goal-Directed Trajectory Prediction and Optimization}

\author{%
  Tom Jurgenson \\
  EE Department\\
  Technion\\
   \And
  Edward Groshev \\
  Osaro Inc.\\
   \And
  Aviv Tamar \\
  EE Department\\
  Technion\\
}

\begin{document}

\maketitle

\begin{abstract}
\input{abstract.tex}
\end{abstract}
\section{Introduction}
\input{intro.tex}
\section{Problem Formulation}
\input{formulation.tex}
\section{Imitation Learning with Sub-Goal Trees}\label{sec:goal-based-il}
\input{trajectory_prediction.tex}

\section{Trajectory Optimization with Sub-Goal Trees}\label{s:rl}
\input{RL.tex}
\section{Related work}
\input{related.tex}
\vspace{-1em}
\section{Experiments}\label{sec:experiments}
\input{experiments.tex}
\vspace{-1em}
\section{Discussion}
\input{discussion.tex}

\bibliographystyle{abbrv}
\bibliography{main}

\pagebreak

\appendix

\section{Imitation Learning - Full Experimental Settings}
\input{appendix_il_experiments.tex}
\end{document}

%% file: abstract.tex
Many AI problems, in robotics and other domains, are goal-directed, essentially seeking a trajectory leading to some goal state. In such problems, the way we choose to represent a trajectory underlies algorithms for trajectory prediction and optimization. Interestingly, most all prior work in imitation and reinforcement learning builds on a \emph{sequential} trajectory representation -- calculating the next state in the trajectory given its predecessors.
We propose a different perspective: a goal-conditioned trajectory can be represented by first selecting an intermediate state between start and goal, partitioning the trajectory into two. Then, recursively, predicting intermediate points on each sub-segment, until a complete trajectory is obtained. We call this representation a \emph{sub-goal tree}, and building on it, we develop new methods for trajectory prediction, learning, and optimization.
We show that in a supervised learning setting, sub-goal trees better account for trajectory variability, and can predict trajectories exponentially faster at test time by leveraging a concurrent computation. Then, for optimization, we derive a new dynamic programming equation for sub-goal trees, and use it to develop new planning and reinforcement learning algorithms. These algorithms, which are not based on the standard Bellman equation, naturally account for hierarchical sub-goal structure in a task.
Empirical results on motion planning domains show that the sub-goal tree framework significantly improves both accuracy and prediction time.

%% file: intro.tex
Many AI problems can be characterized as learning or optimizing goal-directed trajectories of a dynamical system. For example, robotic skill learning seeks trajectories that perform some task, such as hitting a table-tennis ball~\cite{mulling2013learning} or opening a door~\cite{gu2017deep}, while in motion planning a trajectory that reaches some goal without colliding into obstacles is sought~\cite{lavalle2006planning}. For such problems, AI algorithms such as imitation learning (IL) and reinforcement learning (RL) must represent the trajectory in some computational structure, and most prior work has built on a trajectory representation that mirrors the underlying dynamical system: trajectory states evolve in a sequential manner, where the next state is computed based on the previous states~\cite{argall2009survey, sutton1998reinforcement}.

However, the sequential representation has several drawbacks. When learning to imitate trajectories, for example, it is known that accumulating errors can lead to diverging trajectories~\cite{ross2011reduction}, and enforcing the condition that the goal is reached can require additional mechanisms, such as modeling the goal as an attractor in a dynamical system~\cite{ijspeert2013dynamical}. In domains where fast prediction is required, such as robotic motion planning or autonomous driving, the sequential trajectory prediction can be a limiting factor.
In RL, which relies on the sequential Bellman equation, learning to optimize goal-conditioned policies is a challenging problem~\cite{schaul2015universal,andrychowicz2017hindsight}.

In this work, we challenge the conventional representation of a trajectory as a sequential process, and posit that for goal-directed problems, a structure that naturally segments the trajectory into sub-goals is more appropriate.
Our main insight is that a goal-directed trajectory can be represented as a tree structure, where first, an intermediate sub-goal between the initial state and the goal is computed, segmenting the trajectory into two. 
Then, for each segment, an additional sub-goal is computed. 
This divide-and-conquer process continues recursively until a complete trajectory is obtained (see Figure \ref{fig:one_step_vs_trajsplit} for an illustration.). 
We call this representation a \emph{sub-goal tree}, and propose it as a basis for developing prediction and optimization algorithms for both the supervised learning and RL settings.

\begin{figure}
\centering
\includegraphics[width=\textwidth]{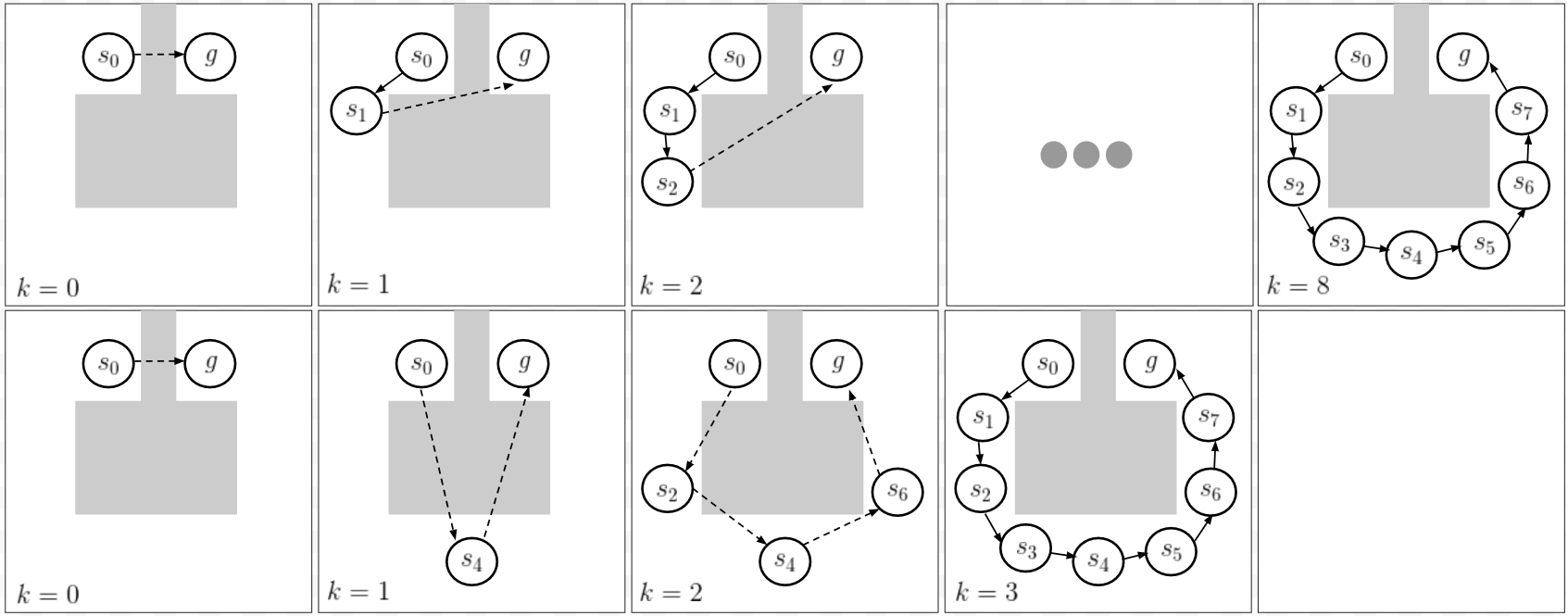}
\caption{Trajectory prediction methods. Upper row: a conventional \textit{Sequential } representation. Lower row: \textit{Sub-Goal Tree} representation. 
Solid arrows indicate predicted segments, while dashed arrows indicate segments that still require to be predicted.
By concurrently predicting sub-goals, a \textit{Sub-Goal Tree} requires $3$ sequential computations, while the \textit{sequential} requires $8$.}
\label{fig:one_step_vs_trajsplit}
\vspace{-5mm}
\end{figure}

Sub-goal tree trajectories are much faster to predict at test time, since sub-goals can be predicted concurrently for multiple segments in a trajectory, a feature that is important for domains where fast prediction is required. 
Additionally, sub-goal trees can be used as an \emph{anytime} prediction algorithm, as the sub goals provide a meaningful approximation of the trajectory from start to goal. In the supervised learning setting, the sub-goal tree representation naturally leads to an algorithm that learns to predict sub-goals in a trajectory. We show that this approach can better account for multi-modal variations in the trajectory, and leads to improved accuracy. 

In the optimization setting, we consider the all-pairs shortest path problem on a graph (APSP;\cite{lavalle2006planning}), and propose a new dynamic programming (DP) principle for APSP based on sub-goal trees. 
Then, based on this formulation, we develop an approximate DP (a.k.a. RL) algorithm for continuous-state goal-based problems, which naturally accounts for any hierarchical sub-goal structure in the problem.
We show that this algorithm can learn significantly better than conventional goal-based RL algorithms that rely on Bellman's sequential DP equation~\cite{bertsekas2005dynamic}. 
To the best of our knowledge, this is the first APSP algorithm that can work with function approximation.

%% file: formulation.tex
In this work we are interested in learning and optimizing trajectories of dynamical systems. We present a uniform formulation based on controlled dynamical systems.
Consider a deterministic controlled dynamical system in discrete time, defined over a state space $S$ and an action space $U$: 
\begin{equation}\label{eq:dynamical_system}
    s_{t+1} = f(s_t, u_t),
\end{equation}
where $s_t, s_{t+1}\in S$ are the states at time $t$ and $t+1$, respectively, $u_t \in U$ is the control at time $t$, and $f$ is a stationary state transition function.\footnote{We leave the extension of our ideas to stochastic and time varying systems for future work.}

In the Imitation Learning (IL) setting, our goal is to imitate trajectories of \eqref{eq:dynamical_system} where actions are selected by an expert. Specifically, we are given a dataset $D^{\pi^*}=\left\{ \tau_i\right\}_{i=1}^N$ of $N$ trajectory demonstrations, provided by an expert policy $\pi^*$.
Each trajectory demonstration $\tau_i$ from the dataset $D^{\pi^*}$, is a sequence of $T_i$ states, i.e. $\tau_i = s_0^i, s_1^i \dots s_{T_i}^i$.\footnote{We note that limiting our discussion to states only can be easily extended to include actions as well by concatenating states and actions. However, we refrain from that in this work in order to simplify the notations. } 
In this work, we assume a goal-based setting, that is, we assume that the expert policy generates  trajectories that lead to a goal state which is the last state in each trajectory demonstration. Our goal is to learn a policy that, given a pair of current state and goal state $(s,g)$, predicts the trajectory that $\pi^*$ would have chosen to reach $g$ from $s$.

In many practical settings, however, expert demonstrations are difficult to obtain. In such cases we resort to an RL setting, where an optimal policy has to be learned from data. In this case, we extend the dynamical system \eqref{eq:dynamical_system} to include costs. Given some initial state $s_0$ and a goal $g$, an optimal trajectory reaches the goal while minimizing the costs:
\begin{equation}\label{eq:opt_control}
        \min_{T, u_0,\dots,u_{T}} \sum_{t=0}^T c(s_t, u_t), 
        \quad s.t. \quad {s_{t + 1}} = {f}({s_t},{u_t}), 
        \quad {s_{T}} = g,
\end{equation}
where $c(s,u)$ is a non-negative cost function. We assume that the dynamics function $f$ is not known in advance. Instead, we are given a data set of transitions $D = \left\{ s, u, c, s'\right\}$, where $s'=f(s,u)$ and $c=c(s,u)$, sampled from some arbitrary distribution. From this data, our goal is to learn a goal-conditioned policy $\pi(s,g) : S\times S \to U$ that, when used to select actions in the system, produces an optimal trajectory. This setting is also known as batch-RL~\cite{lagoudakis2003least,ernst2005tree,riedmiller2005neural}. 

%% file: trajectory_prediction.tex
In this section, we consider the imitation learning problem, and show that a method based on sub-goal trees is more efficient than the conventional approach.
We focus on the Behavioral Cloning~(BC)\cite{pomerleau1989alvinn} approach to IL, where a parametric model for a policy $\hat{\pi}$ with parameters $\theta$ is learned by maximizing the log-likelihood of observed trajectories in $D^{\pi^*}$, i.e.,
\begin{align}\label{eq:bc-max-likelihood-general}
    \theta^* = \arg\max_\theta { \mathbb{E}_{\tau_i\sim D^{\pi^*}}\left[
        \log P_{\hat{\pi}}(\tau_i=s_0^i,s_1^i, \dots s_T^i|s_0^i,s_{T}^i; \theta)
    \right] }.
\end{align}
Denote the horizon $T$ as the maximal number of states in a trajectory. 
For ease of notation we assume $T$ to be the same for all trajectories\footnote{For trajectories with $T_i < T$ assume $s_i^{T_i}$ repeats $T-T_i$ times, alternatively, generate middle states from data until $T_i=T$}. 
Also, let $s_i=s_0^i$ and $g_i=s_{T_i}^i$ denote the start and goal states for $\tau_i$.
We ask -- 
how to best represent the distribution $P_{\hat{\pi}}(\tau|s,g;\theta)$?

The \textit{sequential} trajectory representation~\cite{pomerleau1989alvinn,ross2011reduction,zhang2018auto,qureshi2018motion}, a popular approach, decomposes $P_{\hat{\pi}}(s_0,s_1, \dots s_T|s,g;\theta)$ by sequentially predicting states in the trajectory conditioned on previous predictions.
Concretely, let $h_t=s_0,s_1,\dots,s_t$ denote the history of the trajectory at time index $t$, the decomposition assumed by the \textit{sequential} representation is $P_{\hat{\pi}}(s_0,s_1, \dots s_T|s,g;\theta) = P_{\hat{\pi}}(s_1 | h_0, g ; \theta) P_{\hat{\pi}}(s_2 | h_1, g;\theta)\dots P_{\hat{\pi}}(s_T | h_{T-1}, g ; \theta)$.
Using this decomposition, \eqref{eq:bc-max-likelihood-general} becomes:

\begin{align}\label{eq:bc-max-likelihood-sequential}
    \theta^* = \arg\max_\theta { \mathbb{E}_{\tau_i\sim D^{\pi^*}}\left[
        \sum_{t=1}^{T}{\log{P_{\hat{\pi}}(s_{t+1}^i|h_t^i,g^i; \theta)}}
    \right] }.
\end{align}
We can learn $P_{\hat{\pi}}$ using a batch stochastic gradient descent (SGD) method.
To generate a sample $(s_t, h_{t-1}, g)$ in a training batch, a trajectory $\tau_i$ is sampled from $D^{\pi^*}$, and an observed state $s_t^i$ is further sampled from $\tau_i$. 
Next, the history $h_{t-1}^i$ and goal $g_i$ are extracted from $\tau_i$.
After learning, sampling a trajectory between $s$ and $g$ is a straight-forward iterative process, 
where the first prediction is given by $s_1 \sim P_{\hat{\pi}}(s|h_0=s,g; \theta)$, and every subsequent prediction is given by $s_{t+1} \sim P_{\hat{\pi}}(s|h_t=(s,s_1,\dots s_t),g; \theta)$.
This iterative process stops once some stopping condition is met (such as a target prediction horizon is reached, or a prediction is 'close-enough' to $g$).
Pseudo-code for learning and prediction is further detailed in the supplementary material.



This sequential decomposition has two major drawbacks. First, since prediction at time $t$ requires the history $h_{t-1}$, the computation time of predicting an entire trajectory is linear in the horizon $T$. 
Furthermore, as each prediction accumulates error, the trajectory can drift significantly away from the demonstrations~\cite{ross2011reduction}.
Both drawbacks could be mitigated by using an alternative non-sequential decomposition of the trajectory distribution, as we describe next.
\textbf{The Sub-Goal Tree Representation:}
We propose an alternative representation for a trajectory, based on a 'divide-and-conquer' approach.
Let $\tau(s_1,s_2)$ denote the sub-trajectory of $\tau$ starting in $s_1$ and ending in $s_2$.
We note that for any mid-point $s_m$ in a trajectory $\tau$ with start state $s$ and goal state $g$, $\tau$ can be written as a concatenation of two sub-trajectories: $[\tau(s,s_m), \tau(s_m,g)]$.\footnote{Here and throughout the rest of the paper, the concatenation $[\tau(s,s_m), \tau(s_m,g)]$ is understood to contain the midpoint $s_m$ only once.} 
To simplify notation, we consider the case where the horizon $T$ is a power of two.
\footnote{To handle a general horizon, one can repeat elements to obtain a power of 2 horizon. Alternatively, interpolation between states can be performed.}
We therefore choose to decompose the probability of a trajectory recursively as (cf. Figure~\ref{fig:one_step_vs_trajsplit})  
$
    P_{\hat{\pi}}(s_0,s_1, \dots s_T|s,g;\theta) = P_{\hat{\pi}}(s_{T/2} | s, g ; \theta) P_{\hat{\pi}}(s_{T/4} | s, s_{T/2};\theta)P_{\hat{\pi}}(s_{3T/4} | s_{T/2},g;\theta)\dots
$,
the recursion ends when the start and stop indices are equal and we set the probability to $P_{\hat{\pi}}(s_t | s_t, s_t ; \theta)=1$.
Applying this decomposition on Eq.\ref{eq:bc-max-likelihood-general} results in: 
\begin{equation}\label{eq:bc-max-likelihood-sub-goal}
    \!\theta^* \!\!=\!\! \arg\max_\theta { \mathbb{E}_{\tau_i\sim D^{\pi^*}}\!\!\!\left[
        \log{P_{\hat{\pi}}(s_{T/2}^i | s^i, g^i ; \theta)}\!\!+\!\log{P_{\hat{\pi}}(s_{T/4}^i | s^i, s_{T/2}^i;\theta)}\!\!+\!\log{P_{\hat{\pi}}(s_{3T/4}^i | s_{T/2}^i,g^i;\theta)}\! \dots 
    \right] }
\end{equation}

To organize our data for optimizing Eq.\ref{eq:bc-max-likelihood-sub-goal}, we first sample a trajectory $\tau_i$ from $D^{\pi^*}$ for each sample in the batch. From $\tau_i$ we sample two states $s_a^i$ and $s_b^i$ and obtain their midpoint $s_{\frac{a+b}{2}}^i$.
Pseudo-code for sub-goal tree learning is provided in Section \ref{supp:algorithms} of the supplementary material. 
A clear advantage of Sub-goal trees over sequential representations is in prediction, which can exploit parallel computation. In predicting a sub-goal tree between $s$ and $g$, we first predict a midpoint $s_m$. 
Then, the predictions of midpoints for the two segments $(s,s_m)$ and $(s_m,g)$ are \textit{independent}, and can be computed concurrently. 
This computation can be described by a binary tree rooted at $\tau(s,g)$.
Recursively, at each level $k$ of the tree we can predict $2^k$ midpoints concurrently, resulting in an exponential speedup in prediction time (requiring $O(\log(T))$ to predict $T$-horizoned trajectory, as opposed to $O(T)$ for sequential prediction). 
Algorithm \ref{alg:sub-goal-sl-predict} shows the prediction process of the \textit{Sub-Goal Tree} approach, and Figure \ref{fig:one_step_vs_trajsplit} shows the difference in prediction strategies.



\begin{algorithm}[htp]
  \SetAlgoLined\DontPrintSemicolon
  \SetKwProg{myalg}{Algorithm}{}{}
  \SetKwFunction{predictSGT}{PredictSGT}
  \myalg{ }{
  \nl Input: parameters $\theta$ of parametric distribution $P_{\hat{\pi}}$, start state $s$, goal state $g$, max depth $K$ \;
  \nl \Return [$s$] + \predictSGT{$\theta$, $s$, $g$, $K$} +[$g$]\;
  }{}
  \setcounter{AlgoLine}{0}
  \SetKwProg{myproc}{Procedure}{}{}
  \myproc{\predictSGT{$\theta$, $s_1$, $s_2$, $k$}}{
  \If{$k>0$}{
  \nl Predict midpoint $s_m \sim P_{\hat{\pi}}(s_m |s_1, s_2;\theta)$\;
  \nl \Return\predictSGT{$\theta$, $s_1$, $s_m$, $k-1$} + [$s_m$] + \predictSGT{$\theta$, $s_m$, $s_2$, $k-1$} \;
  }
  }
\caption{Sub-Goal Tree BC Trajectory Prediction }\label{alg:sub-goal-sl-predict}
\end{algorithm}

%% file: RL.tex
We next discuss how to solve the batch-RL problem using sub-goal trees. We start from an underlying dynamic programming principle and then proceed to develop learning algorithms.

\subsection{A Dynamic Programming Principle for Sub-Goal Trees}

In this section we derive a dynamic programming principle for goal-conditioned trajectory optimization based on sub-goal trees. 
To simplify our derivation, we restrict ourselves to a discrete-state formulation of Problem \eqref{eq:opt_control}, also known as the all-pairs shortest path (APSP) problem on a weighted graph. 
We will later extend our approach to the continuous case using function approximation. 

Consider a directed and weighted graph with $N$ nodes $s_1,\dots, s_N$, and denote by $c(s,s')\geq 0$ the weight of edge $s\to s'$.\footnote{Technically, and similarly to standard APSP algorithms~\cite{russel2010AI}, we only require that there are no negative cycles in the graph. To simplify our presentation, however, we restrict $c$ to be non-negative. } To simplify notation, we replace unconnected edges by edges with weight $\infty$, creating a complete graph. The APSP problem seeks the shortest paths (i.e., a path with minimum sum of costs) from any start node $s$ to any goal node $g$ in the graph. Note the similarity to Problem \eqref{eq:opt_control}, where feasible transitions of the dynamical system are now represented by edges between states in the graph. We next derive an APSP algorithm based on sub-goal trees.

Let $V_k(s, s')$ denote the shortest path from $s$ to $s'$ in $2^k$ steps or less. Note that by our convention about unconnected edges above, if there is no such trajectory in the graph then $V_k(s, s')=\infty$. We observe that $V_k$ obeys a dynamic programming relation, which we term \emph{sub-goal tree dynamic programming} (STDP), as established next.
\begin{theorem}
Consider a weighted graph with $N$ nodes and no negative cycles. Let $V_k(s, s')$ denote the cost of the shortest path from $s$ to $s'$ in $2^k$ steps or less, and let $V^*(s,s')$ denote the cost of the shortest path from $s$ to $s'$. Then, $V_k$ can be computed according to the following equations: 
\begin{equation}\label{eq:DP_trajsplit}
    \begin{split}
        V_0(s, s') &= c(s, s'), \quad \forall s,s': s\neq s'; \\
        V_k(s, s) &= 0, \quad \forall s; \\
        V_k(s, s') &= \min_{s_m} \left\{ V_{k-1}(s, s_m) + V_{k-1}(s_m, s')\right\}, \quad \forall s,s' : s\neq s'.
    \end{split}
\end{equation}
Furthermore, for $k \geq \log_2(N)$ we have that $V_k(s, s') = V^*(s, s')$ for all $s,s'$. 
\end{theorem}

\begin{proof}
First, by definition, the shortest path from $s$ to itself is $0$. In the following, therefore, we assume that $s\neq s'$. We will show by induction that each $V_k(s,s')$ in Algorithm \eqref{eq:DP_trajsplit} is the cost of the shortest path from $s$ to $s'$ in $2^k$ steps or less.

Let $\tau_k(s,s')$ denote a shortest path from $s$ to $s'$ in $2^k$ steps or less, and let $V_k(s,s')$ denote its corresponding cost. 
Our induction hypothesis is that $V_{k-1}(s, s')$ is the cost of the shortest path from $s$ to $s'$ in $2^{k-1}$ steps or less. We will show that $V_k(s, s') = \min_{s_m} \left\{ V_{k-1}(s, s_m) + V_{k-1}(s_m, s')\right\}$.
Assume by contradiction that there was some $s^*$ such that $V_{k-1}(s, s^*) + V_{k-1}(s^*, s') < V_k(s,s')$. Then the concatenated trajectory $[\tau_{k-1}(s,s^*), \tau_{k-1}(s^*,s')]$ would have $2^k$ steps or less,
contradicting the fact that $\tau_k(s,s')$ is a shortest path from $s$ to $s'$ in $2^k$ steps or less. So we have that $V_k(s, s') \leq \left\{ V_{k-1}(s, s_m) + V_{k-1}(s_m, s')\right\} \quad \forall s_m$. 
Since the graph is complete, $\tau_k(s,s')$ can be split into two trajectories of length $2^{k-1}$ steps or less. Let $s_m$ be a midpoint in such a split. 
Then we have that $V_k(s, s') = \left\{ V_{k-1}(s, s_m) + V_{k-1}(s_m, s')\right\}$. 
So equality can be obtained and thus we must have $V_k(s, s') = \min_{s_m} \left\{ V_{k-1}(s, s_m) + V_{k-1}(s_m, s')\right\}$.
To complete the induction argument, we need to show that $V_0(s, s') = c(s, s')$. This holds since for $k=0$, for each $s,s'$, the only possible trajectory between them of length $1$ goes through the edge $s,s'$. 

Finally, since there are no negative cycles in the graph, for any $s,s'$, the shortest path has at most $N$ steps. Thus, for $k= \log_2 (N)$, we have that  $V_k(s, s')$ is the cost of the shortest path in $N$ steps or less, which is the shortest path between $s$ and $s'$.
\end{proof}

The STDP algorithm is built on the sub-goal tree idea. To see this, note that in Eq.~\eqref{eq:DP_trajsplit}, the minimization over $s_m$ is effectively a search for the next sub-goal in the sub-goal tree. Furthermore, predicting a sub-goal tree of depth $k$ between a pair of states $(s,g)$ can be done by recursively computing sub-goals according to Eq.~\eqref{eq:DP_trajsplit}, starting from $V_k(s,g)$. 
The main benefit of STDP however, is that it lends to a simple \emph{approximate dynamic programming} (a.k.a. RL) formulation using function approximation, as we show next. 


\subsection{Batch RL with Sub-Goal Trees}
We now describe a batch RL algorithm with function approximation based on the STDP algorithm above. Our approach is inspired by the fitted-Q algorithm for finite horizon Markov decision processes (\cite{tsitsiklis2001regression}; see also~\cite{ernst2005tree,riedmiller2005neural} for the discounted horizon case). 
Assume that we have some estimate $\hat{V}_k(s,s')$ of the value function of depth $k$ in STDP. Then, for any pair of start and goal states $s,g$, we can estimate $\hat{V}_{k+1}(s,g)$ as 
\vspace{-1em}
\begin{equation}\label{eq:STDP_min}
    \hat{V}_{k+1}(s,g) = \min_{s_m} \left\{ \hat{V}_{k}(s, s_m) + \hat{V}_{k}(s_m, g)\right\}
\end{equation} 
Thus, if our data consisted of start and goal pairs, we could use \eqref{eq:STDP_min} to generate \emph{regression targets} for the next value function, and use any regression algorithm to fit $\hat{V}_{k+1}(s,s')$. This is the essence of the approximate STDP algorithm (Algorithm \ref{alg:approximate_STDP}). Since our data does not contain explicit goal states, we simply define goal states to be randomly selected states from within the data.

The first iteration $k=0$ in STDP, however, requires special attention. We need to fit the cost function for neighboring states, yet make sure that states which are not reachable in a single transition have a high cost. To this end, we fit the observed costs $c$ to the observed transitions $s,s'$ in the data, and a high cost $C_{max}$ to transitions from the observed states to randomly selected states.

We also need a method to approximately solve the minimization problem in \eqref{eq:STDP_min}. In our experiments, we discretized the state space and performed a simple grid search. Other methods could be used in general. For example, if $V_k$ is represented as a neural network, then one can use gradient descent. Naturally, the quality of the approximate STDP solution will depend on the quality of solving this minimization problem. 

\begin{algorithm}[htp]
  \SetAlgoLined\DontPrintSemicolon
  \SetKwProg{myalg}{Algorithm}{}{}
  \myalg{ }{
  \nl Input: dataset $D = \left\{ s, u, c, s'\right\}$, Maximum path cost $C_{max}$ \;
  \nl Create transition data set $D_{trans} = \left\{ s, s'\right\}$ and targets $T_{trans} = \left\{ c\right\}$ with $s,s',c$ taken from $D$\;
  \nl Create random transition data set $D_{random} = \left\{ s, s_{rand}\right\}$ and targets $T_{random} = \left\{ C_{max}\right\}$ with $s,s_{rand}$ randomly chosen from states in $D$\;
  \nl Create self transition data set $D_{self} = \left\{ s, s\right\}$ and targets $T_{self} =\left\{ 0\right\}$ with $s$ taken from $D$\;
  \nl Fit $\hat{V}_0(s,s')$ to data in $D_{trans}, D_{random}, D_{self}$ and targets $T_{trans}, T_{random}, T_{self}$\;
  \For{$k: 1...K$}{
  \nl Create goal data set $D_{goal} \!=\! \left\{ s, g\right\}$ and targets 
$T_{goal} \!=\! \{ \min_{s_m} \!\!\{ \hat{V}_{k-1}(s, s_m) \!+\! \hat{V}_{k-1}(s_m, g)\}\}$ with $s,g$ randomly chosen from states in $D$\;
  \nl Fit $\hat{V}_k(s,s')$ to data in $D_{goal}$ and targets in $T_{goal}$\;
  }
  }
\caption{Approximate STDP}
\label{alg:approximate_STDP}
\end{algorithm}

\vspace{-1em}
\paragraph{Why not use the Floyd-Warshall Algorithm?}
At this point, the reader may question why we do not build on the Floyd-Warshall (FW) algorithm for the APSP problem. The FW method maintains a value function $V_{FW}(s,s')$ of the shortest path from $s$ to $s'$, and updates the value using the relaxation $V_{FW}(s,s') := \min \left\{V_{FW}(s,s'), V_{FW}(s,s_m)+V_{FW}(s_m,s')\right\}$. If the updates are performed over all $s_m, s,$ and $s'$ (in that sequence), $V_{FW}$ will converge to the shortest path, requiring $O(N^3)$ computations~\cite{russel2010AI}. One can also perform relaxations in an arbitrary order, as was suggested by Kaelbling~\cite{kaelbling1993learning}, and more recently in~\cite{dhiman2018floyd}, to result in an RL style algorithm. However, as was already observed in~\cite{kaelbling1993learning}, the FW relaxation requires that the values always over-estimate the optimal costs, and any under-estimation error, due to noise or function approximation, gets propagated through the algorithm without any way of recovery, leading to instability. Indeed, both \cite{kaelbling1993learning,dhiman2018floyd} showed results only for table-lookup value functions, and in our experiments we have found that replacing the STDP update with a FW relaxation (reported in the supplementary) leads to instability when used with function approximation. On the other hand, the complexity of STDP is $O(N^3\log N)$, but the explicit dependence on $k$ in the value function allows for a stable update when using function approximation.

%% file: related.tex
Various trajectory representations have been investigated for learning robotic skills~\cite{mulling2013learning, sung2018robobarista} and navigation~\cite{qureshi2018motion}. Dynamical movement primitives~\cite{ijspeert2013dynamical} is a popular approach that represents a continuous trajectory as a dynamical system with an attractor at the goal, and has been successfully used for IL and RL~\cite{kober2013reinforcement,mulling2013learning,peters2008reinforcement}. The temporal segment approach of~\cite{mishra2017prediction}, on the other hand, predicts segments of a trajectory sequentially.
Recently, in the context of video prediction, Jayaraman et al. \cite{jayaraman2019time} proposed to predict salient frames in a goal-conditioned setting by a supervised learning loss that focuses on the `best' frames. This was used to predict a list of sub-goals for a tracking controller. In contrast, we propose to recursively predict sub-goals. Note that the method of \cite{jayaraman2019time} can be combined within our approach to learn the most salient sub-goal tree.

In RL, the idea of sub-goals has mainly been investigated under the options framework~\cite{sutton1999between}. In this setting, the goal is typically fixed (i.e., given by the reward in the MDP), and useful options are discovered using some heuristic such as bottleneck states~\cite{mcgovern2001automatic, menache2002q} or changes in the value function~\cite{konidaris2012robot}. Universal value functions~\cite{schaul2015universal,andrychowicz2017hindsight} learn a goal-conditioned value function using the Bellman equation. In contrast, in this work we propose a principled motivation for sub-goals based on the APSP problem, and develop a new RL formulation based on this principle. A connection between RL and the APSP has been suggested in~\cite{kaelbling1993learning,dhiman2018floyd}, based on the Floyd-Warshall algorithm. However, as discussed in Section~\ref{s:rl}, these approaches become unstable once function approximation is introduced. 


%% file: experiments.tex
\begin{figure*}
\centering
    \hfill
    \begin{subfigure}[b]{0.22\linewidth}
    \includegraphics[width=.97\textwidth]{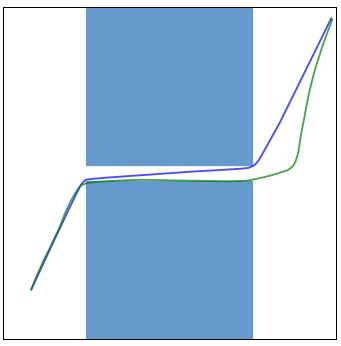}
    \caption{}
    \label{fig:easy}
  \end{subfigure}\hfill
  \begin{subfigure}[b]{0.22\linewidth}
    \includegraphics[width=.97\textwidth]{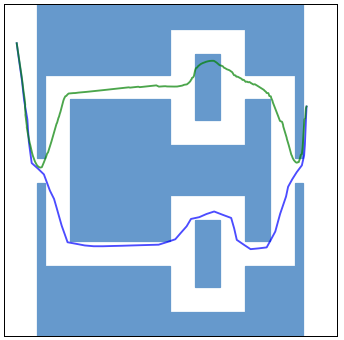}
    \caption{}
    \label{fig:hard}
  \end{subfigure}\hfill
  \begin{subfigure}[b]{0.22\linewidth}
    \includegraphics[width=1.0\textwidth]{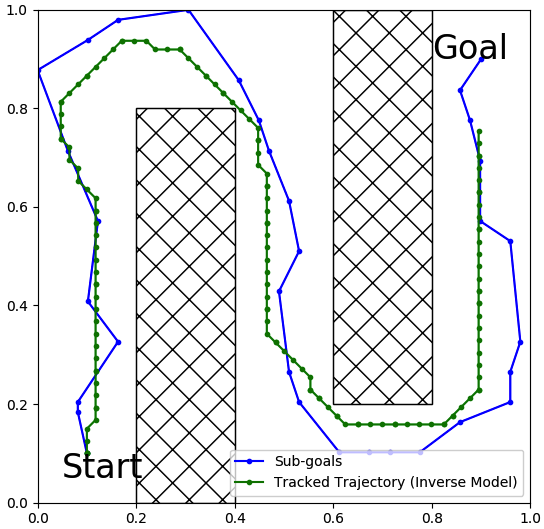}
    \caption{}
    \label{fig:rl}
  \end{subfigure}\hfill
  \begin{subfigure}[b]{0.22\linewidth}
    \includegraphics[width=0.93\textwidth]{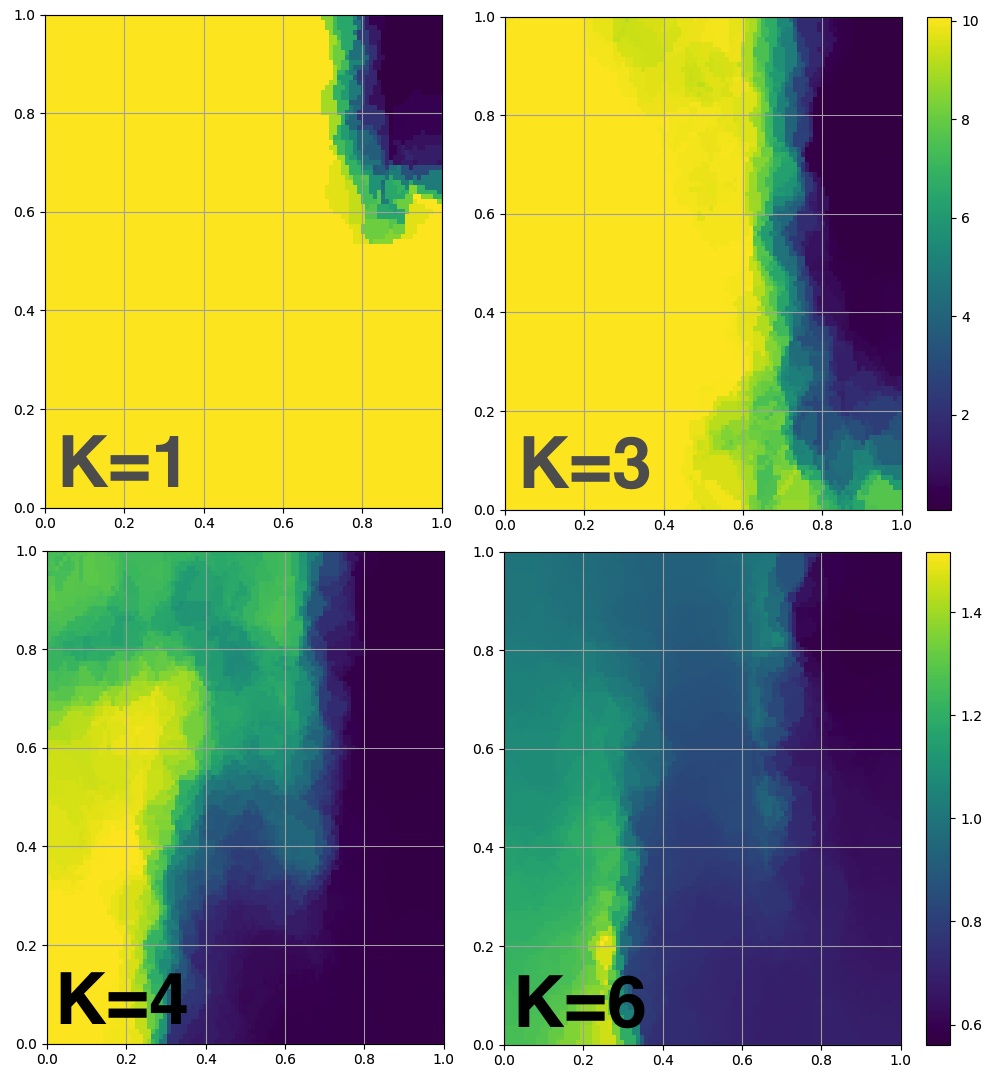}
    \caption{}
    \label{fig:rl_values}
  \end{subfigure}\hfill
\caption{Experiment domains and results. (a+b) IL results. The \textit{simple} domain (a) and \textit{hard} domain (b). A point robot should move only on the white free-space from a room on one side to the room on the other side while avoiding the blue obstacles. SGT plan (blue) executed successfully, \textit{sequential} plan (green) collides with the obstacles. (c+d) RL results. In the domain (c) a robot needs to navigate between the (hatched) obstacles. We show an example sub-goal tree (blue) and trajectory tracking it using an inverse model (green). In (d) we show the approximate value $\hat{V}_k(s,g=[0.9,0.9])$ for several values of $k$. Note how the reachable region to the goal (non-yellow) grows with $k$.}
\label{fig:scenarios}
\vspace{-5mm}
\end{figure*}
\vspace{-.5em}
We report our results for imitation learning and reinforcement learning in a motion planning scenario.
\vspace{-1.5em}
\paragraph{Imitation Learning Experiments:}
We compare the \textit{sequential} and \textit{Sub-Goal Tree}(SGT) approaches for BC.
Consider a point-robot motion-planning problem in a 2D world, with two obstacle scenarios termed \textit{simple} and \textit{hard}, as shown in Figure \ref{fig:scenarios}. In \textit{simple}, the distribution of possible motions from left to right is uni-modal, while in \textit{hard}, at least 4 modalities are possible.

For both scenarios we collected a set of 111K (100K train + 10K validation + 1K test)  expert trajectories from random start and goal states using a state-of-the-art motion planner (OMPL's\cite{sucan2012the-open-motion-planning-library} Lazy Bi-directional KPIECE \cite{csucan2009kinodynamic} with one level of discretization). 


To account for different trajectory modalities, we chose a Mixture Density Network (MDN)\cite{bishop1994mixture} as the parametric distribution of the predicted next state, both for the \textit{sequential} and the SGT representations. 
We train the MDN by maximizing likelihood using Adam~\cite{kingma2014adam}.
To ensure the same model capacity for both representations, we used the same network architecture, and both representations were trained and tested with the same data. Since the dynamical system is Markovian, the current and goal states are sufficient for predicting the next state in the plan, so we truncated the state history in the sequential model's input to contain only the current state.

For \textit{simple}, the MDNs had a uni-modal multivariate-Gaussian distribution, while for \textit{hard}, we experimented with 2 and 4 modal multivariate-Gaussian distributions, denoted as \textit{hard-2G}, and \textit{hard-4G}, respectively.
As we later show in our results the SGT representation captures the demonstrated distribution well even in the \textit{hard-2G} scenario, by modelling a bi-modal sub-goal distribution.

We evaluate the models using unseen start-goal pairs taken from the test set. 
To generate a trajectory, we follow Algorithms \ref{alg:sequential-sl-predict} and \ref{alg:sub-goal-sl-predict}, and connect the points using linear interpolation. 
We call a trajectory successful if it does not collide with an obstacle en-route to the goal. 
For a failed trajectory, we further measure the \textit{severity} of collision by the percentage of the trajectory being in collision.


Table \ref{tab:il-experiments_results} summarizes the results for both representations. The SGT representation is superior in all three evaluation criteria - motion planning success rate, trajectory prediction times (total time in seconds for the 1K trajectory predictions), and severity.
Upon closer look at the two \textit{hard} scenarios, the \textit{Sub Goal Tree} with a bi-modal MDN outperforms \textit{sequential} with 4-modal MDN, suggesting that the SGT trajectory decomposition better accounts for multi-modal trajectories.
Finally, the supplementary section contains further experiments considering additional baselines for the BC case.

\begin{table}[]
\centering
\begin{tabular}{l|ll|ll|ll}
\multirow{2}{*}{} & \multicolumn{2}{c|}{\textbf{Success Rate (\%)}} & \multicolumn{2}{c|}{\textbf{Prediction Times (Seconds)}} & \multicolumn{2}{c}{\textbf{Severity (\%)}} \\
                  & Sequential           & SGT            & Sequential               & SGT                 & Sequential           & SGT       \\ \hline
Simple            & 0.541                & \textbf{0.946}           & 487.179                  & \textbf{28.641}               & \textbf{0.0327}      & 0.0381              \\
Hard - 2G         & 0.013                & \textbf{0.266}           & 811.523                  & \textbf{52.22}                & 0.0803               & \textbf{0.0666}     \\
Hard - 4G         & 0.011                & \textbf{0.247}           & 1052.62                  & \textbf{53.539}               & 0.0779               & \textbf{0.0362}    
\end{tabular}
\caption{Results for IL experiments. 
}
\label{tab:il-experiments_results}
\vspace{-5mm}
\end{table}

\textbf{Batch RL Experiments:}
Next, we evaluate the approximate STDP algorithm. We consider a 2D particle moving in an environment with obstacles, as shown in Figure~\ref{fig:rl}. The particle can move a distance of $0.025$ in one of the eight directions, and suffers a constant cost of $0.025$ in free space, and a large cost of $10$ on collisions. The task is to reach from any starting point to within a $0.15$ distance of any goal point without colliding. This simple domain is a continuous-state optimal control Problem~\eqref{eq:opt_control}, and for start and goal positions that are distant, as shown in Figure~\ref{fig:rl}, it requires long-horizoned planning, making it suitable for studying batch RL algorithms.

To generate data, we sampled states and actions uniformly and independently, resulting in $125$K $(s,u,c,s')$ tuples. As function approximation, we opted for simplicity, and used K-nearest neighbors (KNN) for all our experiments, with K$_{neighbors}=5$. To solve the minimization over states in approximate STDP, we discretized the state space and searched over a $50\times 50$ grid of points.

A natural baseline in this setting is fitted-Q iteration~\cite{ernst2005tree, riedmiller2005neural}. We  verified that for a fixed goal, fitted-Q obtains near perfect results with our data. Then, to make it goal-conditioned, we used a universal Q-function~\cite{schaul2015universal}, requiring only a minor change in the algorithm (see supplementary for pseudo-code). 

To evaluate the different methods, we randomly chose 200 start and goal points, and measured the distance from the goal that the policies reach, and whether they collide with obstacles or not along the way. For fitted-Q, we used the greedy policy with respect to the learned Q function. The approximate STDP method, however, does not automatically provide a policy. Thus, we experimented with two methods for extracting a policy from the learned sub-goal tree. The first is training an inverse model $f_{IM}(s,s')$ -- a mapping from $s,s'$ to $u$, using our data, and the same KNN function approximation. To reach a sub-goal $g$ from state $s$, we simply run $f_{IM}(s,g)$ until we are close enough to $g$ (we set the threshold to $0.15$). An alternative method is first using fitted-Q to learn a goal-based policy, as described above, and then running this policy on the sub-goals. The idea here is that the sub-goals learned by approximate STDP can help fitted-Q overcome the long-horizon planning required in this task.
Note that \emph{all methods use exactly the same data, and the same function approximation} (KNN), making for a fair comparison.

In Table \ref{tab:RL-experiments_results} we report our results. Fitted-Q did not succeed in reaching any but the very closest goals, resulting in a high average distance to goal. Approximate STDP, on the other hand, computed meaningful sub-goals for almost all test cases, resulting in a low average distance to goal when tracked by the inverse model. Figure~\ref{fig:scenarios} shows an example sub-goal tree and a corresponding tracked trajectory. The fitted-Q policy did learn not to hit obstacles, resulting in the lowest collision rate. This is expected, as colliding leads to an immediate high cost, while the inverse model is not trained to take cost into account. Interestingly, combining the fitted-Q policy with the sub-goals improves both long-horizon planning and short horizoned collision avoidance. In Figure \ref{fig:rl_values} we plot the approximate value function $\hat{V}_k$ for different $k$ and a specific goal. Note how the reachable parts of the state space to the goal expand with $k$.
\begin{table}[]
\centering
\begin{tabular}{l|c|c}
                                   & Average Distance to Goal & Average Collision Rate \\ \hline
Sub-goal Tree + Inverse Model      & 0.13                     & 0.25                   \\
Sub-goal Tree + Q-value controller & 0.29                     & 0.06                   \\
Q-value controller                 & 0.58                     & 0.02                  
\end{tabular}
\caption{Results for RL-based controllers. 
}
\label{tab:RL-experiments_results}
\vspace{-8mm}
\end{table}

%% file: discussion.tex
\vspace{-1em}
We have shown that by viewing a trajectory as a hierarchical composition of sub-goals, improved learning and optimization algorithms can be derived. We believe that these ideas would be important for robotics and autonomous driving, and other domains where fast predictions are important.

Our novel approach to RL based on the APSP problem raises many directions for future research. The optimization over sub-goal states motivates an actor-critic approach, where an actor would learn to predict the sub-goals. 
It is interesting whether our approach could be extended to high-dimensional state observations, such as in goal-based video prediction. Intuitively, predicting some features of sub-goals along the trajectory seems like a natural step before predicting the full video trajectory. Finally, we believe that stochastic systems can also be handled by learning the \emph{probability} of reaching different sub-goals.


%% file: appendix_il_experiments.tex
In this appendix we summarize all the technical details required to reproduced the imitation learning results discussed in Section \ref{sec:experiments}.
We start by providing a detailed description of our neural network structure and training process (we use the same structure for all the scenarios and all the representation variants).

The network consists of 4 Fully-Connected hidden layers connecting the inputs with the outputs (the outputs are the distribution parameters defined by the MDN according to the scenario: 1, 2 and 4 multivariate-Gaussians for the \textit{simple}, \textit{hard-2N} and \textit{hard-4N} respectively). 
We used Rectified-Linear Units\cite{DBLP:journals/corr/abs-1803-08375} (ReLU) as our activation function.

The training procedure used Adam\cite{DBLP:journals/corr/KingmaB14} with a batch of 50, and a learning rate of 0.001 that decays by a factor of 0.8 after the validation loss does not improve after 6 consecutive test times. We also defined a minimal learning rate of 0.00001.
We also defined gradient clipping - gradients with L2-norm greater 200 are re-normalized to 200.

We further wanted to emphasize the following details:
\begin{enumerate}
    \item We always connect the an endpoint to the nearest endpoints. Basically, this means that for the sequential model we always connect the last prediction to the goal state. This allows us to circumvent the issue that sequential models have with not reaching to the goal state directly. 
    
    \item The models reported in the main text continue to make prediction even after collision. This gives them the chance to "recover" by minimizing the distance traveled in collision. This also allows us to have the \textit{severity} metric to further analyze which model is beneficial for which scenarios.
    
    \item Early stopping: we use the 10K validation trajectories to evaluate the loss of the model during training for trajectories not included in the training set. Every time we discover a model that scores a lower validation loss we save it.
    The resulting model is the one with lowest validation loss.
    
    \item Sampling: our models are a mixture of Gaussians, however, during test time we are only interested in the safest trajectory a model can find. Thus, to sample from the distribution, we sample a mode (i.e. a specific Gaussian from the mixture) and we take the mean prediction for that Gaussian.
\end{enumerate}

\section{Behavioral Cloning Algorithms}\label{supp:algorithms}

In this appendix we explicitly write the BC algorithms used in Section \ref{sec:goal-based-il} that are not in the main text.
We start with the algorithms for the \textit{sequential} approach.
Algorithm \ref{alg:sequential-sl-train} first provides a sequential training method based on Stochastic Gradient Descent.
Next, Algorithm \ref{alg:sequential-sl-predict}, provides a prediction algorithm  that predicts the next state conditioned on the current prediction and the goal.

Next, we explicitly describe the pseudo-code for training a \textit{Sub-Goal Tree} in Algorithm \ref{alg:sub-goal-sl-train}, and we also provide an inference pseudo-code that emphasizes parallelism as opposed to the recursive version provided in the main text (Algorithm \ref{alg:sub-goal-sl-predict-parallel}).

\begin{algorithm}[htp]
  \SetAlgoLined\DontPrintSemicolon
  \SetKwProg{myalg}{Algorithm}{}{}
  \myalg{}{
  \nl Input: dataset $D=\left\{ \tau_i = s_0^i, s_1^i \dots s_{T_i}^i \right\}_{i=1}^N$, train steps $M$, batch size $B$\;
  \nl Initialize parameters $\theta$ for parametric distribution $P_{\hat{\pi}}(s_m |s, g;\theta)$\;
  \For{$i: 1...M$ }{
  \nl Sample batch of size $B$, each sample: $\tau_i \sim D$, $s_t \sim \tau_i$\;
  \nl $g\leftarrow$ Get goal state according to $\tau_i$ for all items in batch\;
  \nl $h_{t-1}\leftarrow$ Get history of $s_t$ according to $\tau_i$ for all items in batch\;
  \nl Update $\theta$ by minimizing the negative log-likelihood loss: \begin{align*}
       L = -\frac{1}{B}\cdot\sum_{b=1}^B{\bigtriangledown_{\theta}\log{P_{\hat{\pi}}(s_t^b|h_{t-1}^b, g^b; \theta)}}
  \end{align*}\;
  }
  \Return $\theta$ \;
  }
\caption{Sequential BC SGD-Based Training}\label{alg:sequential-sl-train}
\end{algorithm} 

\begin{algorithm}[htp]
  \SetAlgoLined\DontPrintSemicolon
  \SetKwProg{myalg}{Algorithm}{}{}
  \myalg{}{
  \nl Input: parameters $\theta$ of parametric distribution $P_{\hat{\pi}}$, start state $s$, goal state $g$, max steps $K$ \;
  \nl Initialize empty list $L$\;
  \nl Append $s$ to $L$\;
  \For{$k: 1...K$ }{
  \nl Predict next state $s \sim P_{\hat{\pi}}(s |h, g;\theta)$\;
  \tcc{can add a stopping condition if required} \;
  \nl Append $s$ to $L$\;
  }
  \nl \Return $L$\;
  }
\caption{Sequential BC Trajectory Prediction }\label{alg:sequential-sl-predict}
\end{algorithm}

\begin{algorithm}[htp]
  \SetAlgoLined\DontPrintSemicolon
  \SetKwProg{myalg}{Algorithm}{}{}
  \myalg{}{
  \nl Input: dataset $D=\left\{ \tau_i = s_0^i, s_1^i \dots s_{T_i}^i \right\}_{i=1}^N$, train steps $M$, batch size $B$\;
  \nl Initialize parameters $\theta$ for parametric distribution $P_{\hat{\pi}}(s_m |s, g;\theta)$\;
  \For{$i: 1...M$ }{
  \nl Sample batch of size $B$, each sample: $\tau_i \sim D$, $s_1, s_2 \sim \tau_i$\;
  \nl $s_m\leftarrow$ Get midpoint of $[s_1, s_2]$ according to $\tau_i$ for all items in batch\;
  \nl Update $\theta$ by minimizing the negative log-likelihood loss: \begin{align*}
       L = -\frac{1}{B}\cdot\sum_{b=1}^B{\bigtriangledown_{\theta}\log{P_{\hat{\pi}}(s_m^b|s_1^b, s_2^b; \theta)}}
  \end{align*}\;
  }
  \Return $\theta$ \;
  }
\caption{Sub-Goal Tree BC SGD-Based Training}\label{alg:sub-goal-sl-train}
\end{algorithm} 

\begin{algorithm}[htp]
  \SetAlgoLined\DontPrintSemicolon
  \SetKwProg{myalg}{Algorithm}{}{}
  \myalg{}{
  \nl Input: parameters $\theta$ of parametric distribution $P_{\hat{\pi}}$, start state $s$, goal state $g$, max depth $K$ \;
  \nl Initialize empty lists $q_{next}, q_{current}$\;
  \nl Put segment $[s,g]$ in $q_{current}$\;
  \For{$k: 1...K$ }{
  \tcc{execute while body in parallel for all items in $q_{current}$} \;
  \While{$q_{current}$ not empty}{
  \nl Pop segment $[s_1, s_2]$ from $q_{current}$\;
  \nl Predict midpoint $s_m \sim P_{\hat{\pi}}(s_m |s_1, s_2;\theta)$\;
  \nl Put segment $[s_1, s_m]$ and $[s_m, s_2]$ to $q_{next}$\;
  }
  \nl $q_{current}\leftarrow q_{next}$\;
  \nl Empty $q_{next}$\;
  }
  \nl \Return segments in $q_{current}$\;
  }
\caption{Sub-Goal Tree BC Trajectory Prediction (Parallel version)}\label{alg:sub-goal-sl-predict-parallel}
\end{algorithm}

\section{Baseline RL Algorithms}

\newcommand{\unit}{1\!\!1}

In Algorithm \ref{alg:fitted_Q} we present a goal-based versions of fitted-Q iteration~\cite{ernst2005tree} using universal function approximation~\cite{schaul2015universal}, which we used as a baseline in our experiments.

\begin{algorithm}[htp]
  \SetAlgoLined\DontPrintSemicolon
  \SetKwProg{myalg}{Algorithm}{}{}
  \myalg{ }{
  \nl Input: dataset $D = \left\{ s, u, c, s'\right\}$, Goal reached threshold $\delta$ \;
  \nl Create transition data set $D_{trans} = \left\{ s, u, s'\right\}$ and targets $T_{trans} = \left\{ c\right\}$ with $s,s',c$ taken from $D$\;
  \nl Fit $\hat{Q}(s,u,s')$ to data in $D_{trans}$ \;
  \For{$k: 1...K$}{
  \nl Create random goal data set $D_{goal} = \left\{ s, u, g\right\}$ and targets 
  $$T_{goal} = \left\{ \left\{ c(s,u) + \min_{u'}\hat{Q}(s', u', g)\unit_{||s'-g||>\delta}\right\}\right\}$$ with $s,u,s'$ taken from $D$ and $g$ randomly chosen from states in $D$\;
  \nl Fit $\hat{Q}(s,u,s')$ to data in $D_{goal}$ \;
  }
  }
\caption{Fitted Q with Universal Function Approximation}
\label{alg:fitted_Q}
\end{algorithm}

Next, in Algorithm \ref{alg:approximate_FW} we present an approximate dynamic programming version of Floyd-Warshall RL~\cite{kaelbling1993learning} that corresponds to the batch RL setting we investigate. This algorithm was not stable in our experiments, as the value function converged to zero for all states (when removing the self transition fitting in line 7 of the algorithm, the values converged to a constant value).

\begin{algorithm}[htp]
  \SetAlgoLined\DontPrintSemicolon
  \SetKwProg{myalg}{Algorithm}{}{}
  \myalg{}{
  \nl Input: dataset $D = \left\{ s, u, c, s'\right\}$, Maximum path cost $C_{max}$ \;
  \nl Create transition data set $D_{trans} = \left\{ s, s'\right\}$ and targets $T_{trans} = \left\{ c\right\}$ with $s,s',c$ taken from $D$\;
  \nl Create random transition data set $D_{random} = \left\{ s, s_{rand}\right\}$ and targets $T_{random} = \left\{ C_{max}\right\}$ with $s,s_{rand}$ randomly chosen from states in $D$\;
  \nl Create self transition data set $D_{self} = \left\{ s, s\right\}$ and targets $T_{self} =\left\{ 0\right\}$ with $s$ taken from $D$\;
  \nl Fit $\hat{V}(s,s')$ to data in $D_{trans}, D_{random}, D_{self}$ \;
  \For{$k: 1...K$}{
  \nl Create random goal and mid-point data set $D_{goal} = \left\{ s, g\right\}$ and targets 
  $$T_{goal} = \left\{ \min\left\{ \hat{V}(s, g), \hat{V}(s, s_m) + \hat{V}(s_m, g)\right\}\right\}$$ with $s,s_m,g$ randomly chosen from states in $D$\;
  \nl Create self transition data set $D_{self} = \left\{ s, s\right\}$ and targets $T_{self} =\left\{ 0\right\}$ with $s$ taken from $D$\;
  \nl Fit $\hat{V}(s,s')$ to data in $D_{goal},D_{self}$ \;
  }
  }
\caption{Approximate Floyd Warshall}
\label{alg:approximate_FW}
\end{algorithm}

\section{Additional Baseline for BC}

In this section we provide some additional baselines to the supervised learning experiments described in the main text in Section \ref{sec:experiments}.
We answer these questions:
\begin{enumerate}
    \item What is the contribution of the iterative conditioning in the SGT approach?
    
    \item If the model stops predicting when a collision is encountered (but the prediction are still connected to subsequent endpoint) how would this affect the results?
    
\end{enumerate}


\subsection{Contribution of the iterative prediction of the Sub-Goal Tree}\label{apx:sl-direct-prediction}

Each representation discussed, produces for trajectory generation a list of segments' endpoints. 
The advantage of SGT over the \textit{sequential} approach, is that the order and conditioning on those endpoints allows a concurrent computation to predict exponentially faster.
Taking this view to the extreme, we design the \textit{direct} representation, which attempts to predict all intermediate state of the trajectory directly from the start and goal states.
The advantage of this method is that given a machine capable of computing all the endpoints in parallel, we obtain a prediction model that generates a trajectory in $O(1)$ time units.

\textbf{Model: } Keeping the same notations as in Section \ref{sec:goal-based-il}, in the \textit{direct} approach we directly decompose the trajectory distribution to endpoints predictions that depend only on the start $s$ and goal states $g$, and the index of the endpoint to predict, namely: $P_{\hat{\pi}}(s_0,s_1, \dots s_T|s,g;\theta) = P_{\hat{\pi}}(s_1 | s, g,1 ; \theta) P_{\hat{\pi}}(s_2 | s, g,2;\theta)\dots P_{\hat{\pi}}(s_T | s, g, T ; \theta)$. Applying this decomposition on Eq.\ref{eq:bc-max-likelihood-general} and taking the log results in:
$
    \theta^* = \arg\max_\theta { \mathbb{E}_{\tau_i\sim D^{\pi^*}}\left[
        \log{P_{\hat{\pi}}(s_1^i | s^i, g^i, 1 ; \theta)}+\log{P_{\hat{\pi}}(s_2^i | s^i, g^i,2;\theta)} \dots +\log{P_{\hat{\pi}}(s_T^i | s^i, g^i,T ; \theta)}
    \right] }
$.

This model can be trained in an SGD algorithm similar to the SGT Algorithm \ref{alg:sub-goal-sl-train}. 
The only difference is that the model should now consider more than the mid-state as the target (or label) for the loss function. For instance, if our goal is to create an algorithm that predicts $64$ segments, we should train such model to predict on $63$ intermediate states in every trajectory $\tau_i$ in the data.
Inference for such an algorithm is straight-forward: given $s$ and $g$, we generate a prediction for every index and concatenate.

We investigate the trade-offs between fast prediction and accuracy in the experiments at the end of this appendix.
Our findings suggest that the model is much faster but far less accurate than SGT (due to the different decompositions assumed by each model). As a final note, one might combine both approaches in order to extend the SGT from a binary-tree, to a $k$-ary-tree, predicting more endpoints at each call of the recursive process. Complexity-wise this will only affect the base of the logarithm, changing it from $2$ to $k$, but we leave this investigation for future research.

\subsection{Additional Imitation Learning Experiments}

In this appendix we defined the following variations of our supervised model:
\begin{enumerate}
    \item Stop on collision policy: we wanted to investigate the performance trade-offs of the \textit{sequential} and SGT representations of planning time vs. success rate. For this we define two model variations \textit{sequential:stop-coll} and \textit{SGT:stop-coll}.
    
    \item \textit{direct} prediction - predicting intermediate states directly from the start and goal as described previously. 
\end{enumerate}

We start with the analysis of the 'stop-on-collision'. Table \ref{tab:il-experiments_results-collsions} shows the results for the \textit{sequential:stop-coll}, and \textit{SGT:stop-coll} approaches. As we can see, both methods are faster than their non-stopping variants, although the SGT is only marginally faster. Moreover, we can see that the \textit{severity} scores do worsen. Both of these demonstrate the trade-off of accuracy vs. trajectory prediction times, and motivate our selection that the variants that continue to predict after collision should be considered as the proposed variants.

\begin{table}[]
\begin{tabular}{l|ll|ll|ll}
\multirow{2}{*}{} & \multicolumn{2}{c|}{\textbf{Success Rate (\%)}}                                                                                      & \multicolumn{2}{c|}{\textbf{Prediction Times (Seconds)}}                                                                             & \multicolumn{2}{c}{\textbf{Severity (\%)}}                                                                                           \\
                  & \begin{tabular}[c]{@{}l@{}}Sequential:\\ stop-coll\end{tabular} & \begin{tabular}[c]{@{}l@{}}SGT:\\ stop-coll\end{tabular} & \begin{tabular}[c]{@{}l@{}}Sequential:\\ stop-coll\end{tabular} & \begin{tabular}[c]{@{}l@{}}SGT:\\ stop-coll\end{tabular} & \begin{tabular}[c]{@{}l@{}}Sequential:\\ stop-coll\end{tabular} & \begin{tabular}[c]{@{}l@{}}SGT:\\ stop-coll\end{tabular} \\ \hline
Simple            & 0.541                                                           & 0.946                                                              & 322.347                                                         & 28.542                                                             & 0.4551                                                          & 0.0875                                                             \\
Hard - 2G         & 0.014                                                           & 0.255                                                              & 217.494                                                         & 39.925                                                             & 0.5897                                                          & 0.3479                                                             \\
Hard - 4G         & 0.009                                                           & 0.232                                                              & 252.718                                                         & 53.58                                                              & 0.55                                                            & 0.0362                                                            
\end{tabular}
\caption{Success rates, prediction times (seconds) and severity by representation type for all scenarios for 1K test trajectories.}
\label{tab:il-experiments_results-collsions}
\end{table}

Next, we investigate the other proposed model, the \textit{direct} prediction. The results of this model are in Table \ref{tab:il-direct-prediction}. First, we see that indeed the trajectory prediction time is much shorter. Further, we can see that the model is good in some easy cases, as it was able to predict better than the \textit{sequential} model for the \textit{easy} scenario by a large margin. However, as the prediction problem becomes harder, it seems the extra conditioning of the \textit{sequential} prediction allows it to make better prediction as the \textit{direct} approach attains almost zero success rate.

\begin{table}[]
\begin{tabular}{l|lll}
          & Success Rate (\%) & Prediction Times (Seconds) & Severity (\%) \\ \hline
Simple    & 0.716             & \textbf{13.206}            & 0.0343        \\
Hard - 2G & 0.0               & \textbf{27.345}            & 0.1172        \\
Hard - 4G & 0.004             & \textbf{28.718}            & 0.1195       
\end{tabular}
\caption{Success rates, prediction times (seconds) and severity for the \textit{direct} prediction approach for all scenarios for 1K test trajectories.}
\label{tab:il-direct-prediction}
\end{table}